\documentclass{article}

\usepackage{microtype}
\usepackage{graphicx}
\usepackage{subfigure}
\usepackage{booktabs} 

\usepackage{hyperref}
\usepackage{multirow}
\usepackage{breqn}
\usepackage{amsmath}
\usepackage{amssymb}
\usepackage{amsthm}
\usepackage{bbm}
\usepackage{rotating}

\DeclareMathOperator*{\argmax}{arg\,max}

\newtheorem{lemma}{Lemma}
\newtheorem{definition}{Definition}

\usepackage[accepted]{icml2021}

\icmltitlerunning{Attribution Mask: Filtering Out Irrelevant Features by Recursively Focusing Attention on Inputs of DNNs}

\begin{document}

\twocolumn[
\icmltitle{Attribution Mask: Filtering Out Irrelevant Features By Recursively Focusing Attention on Inputs of DNNs}




\begin{icmlauthorlist}
\icmlauthor{Jae-Hong Lee}{to}
\icmlauthor{Joon-Hyuk Chang}{to}
\end{icmlauthorlist}

\icmlaffiliation{to}{Department of Electronics and Computer Engineering, Hanyang University, Seoul, Republic of Korea}

\icmlcorrespondingauthor{Jae-Hong Lee}{ljh93ljh@hanyang.ac.kr}

\icmlkeywords{Accountability, Transparency, Interpretability}

\vskip 0.3in
]



\printAffiliationsAndNotice{}  

\begin{abstract}
Attribution methods calculate attributions that visually explain the predictions of deep neural networks (DNNs) by highlighting important parts of the input features.
In particular, gradient-based attribution (GBA) methods are widely used because they can be easily implemented through automatic differentiation.
In this study, we use the attributions that filter out irrelevant parts of the input features and then verify the effectiveness of this approach by measuring the classification accuracy of a pre-trained DNN.
This is achieved by calculating and applying an \textit{attribution mask} to the input features and subsequently introducing the masked features to the DNN, for which the mask is designed to recursively focus attention on the parts of the input related to the target label.
The accuracy is enhanced under a certain condition, i.e., \textit{no implicit bias}, which can be derived based on our theoretical insight into compressing the DNN into a single-layer neural network.
We also provide Gradient\,*\,Sign-of-Input (GxSI) to obtain the attribution mask that further improves the accuracy.
As an example, on CIFAR-10 that is modified using the attribution mask obtained from GxSI, we achieve the accuracy ranging from 99.8\% to 99.9\% without additional training. 
\footnote[2]{All code will be available here: \\ \texttt{http://github.com/j-pong/AttentionMask}}
\end{abstract}

\section{Introduction and Related Works}
\begin{figure}[ht]
	\centering
	\includegraphics[scale=0.475]{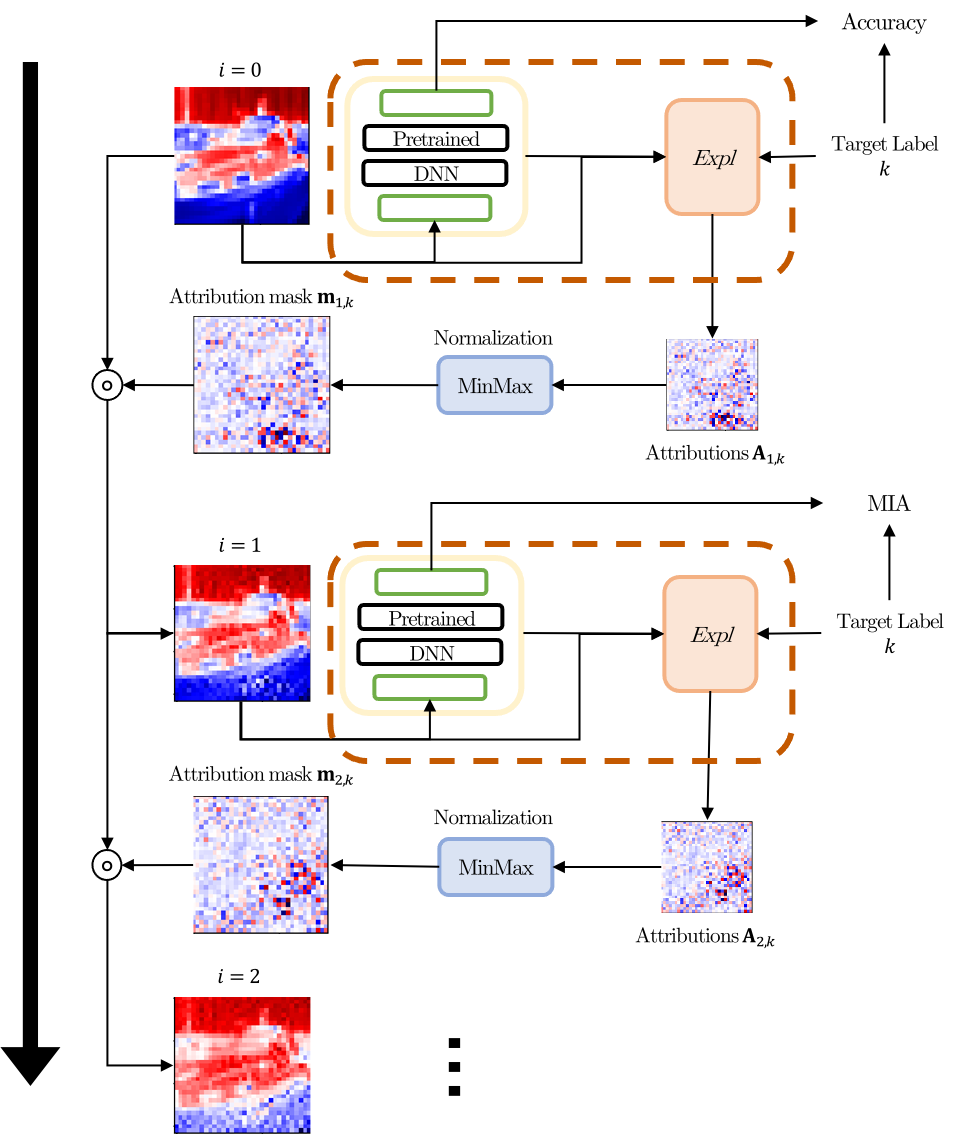}
	\caption{
	Schematic diagram for measuring the masked input accuracy (MIA) of gradient-based attribution (GBA) method, $Expl$. The attributions are transformed to an \textit{attribution mask} via MinMax normalization. The mask is applied to the input features, and the entire process is repeated for $i=1,2,\cdots\,$, as depicted in the figure. See Algorithm \ref{algo:measure} for the full description.
	}
\end{figure}

A deep neural network (DNN) is a powerful function that can express various types of functions given only input features and target labels without a specific model design. 
However, explaining the prediction of a DNN is difficult owing to its complex structure. 
Existing studies \cite{ribeiro2016should, lundberg2017unified, kindermans2017learning, sundararajan2017axiomatic, shrikumar2017learning} have approached this problem using an attribution method that visually explains the prediction of DNNs by highlighting important parts of the input features. In particular, gradient-based attribution (GBA) methods are widely used because they can be easily implemented through automatic differentiation.

Initially, \cite{simonyan2014deep} proposed calculating gradients with respect to the input features of a DNN to obtain a saliency map that highlights important parts of an image. 
Subsequently, Guided Backpropagation \cite{springenberg2014striving} obtained attributions that look similar to the target object by forcing the negative components of the gradient obtained from the DNN to zeros. 
Meanwhile, Gradient\,*\,Input method obtained attributions through an element-wise product between the gradients and the input features \cite{ancona2017towards,kindermans2017learning}. 
In a similar context, Integrated Gradients \cite{sundararajan2017axiomatic} integrate the gradients obtained by gradually changing the input features to the baselines, for which the integrated gradients are multiplied to the difference between the input features and the baselines to calculate the attributions.
In particular, this method satisfies two axioms: sensitivity and implementation invariance. 
However, it has a drawback of requiring numerous computations. 
From a different perspective, Learning Important Features Through Propagating Activation Differences \cite{shrikumar2017learning} introduced DeepLIFT, which has a computational advantage while violating implementation invariance of the aforementioned axioms. 
In addition, DeepSHAP combined DeepLIFT and Shapley values \cite{shapley1953value}, which was proposed by the SHAP \cite{lundberg2017unified}.

These methods were mainly evaluated by judging attributions as visualization.
Alternatively, there are attempts to measure the accuracy of a DNN after applying attributions to input features \cite{hooker2018benchmark, kim2019saliency, chalasani2020concise, shi2020informative, phang2020investigating}. Especially, the ROAR score \cite{hooker2018benchmark} is calculated by retraining the model from scratch through masking relevant image parts to zero. 
Meanwhile, \cite{adebayo2018sanity} proposed a sanity check, 
where the ``randomized" attributions are computed using a DNN, whose parameters of certain layers are randomly assigned, and subsequently compared to the ``original" attributions through a Structural Similarity Index Measure (SSIM).
Moreover, with this sanity check, \cite{sixt2020explanations} found that attributions of GBA methods, except for DeepLIFT, are independent of the parameters of the later layers. 
They also explained how the information of the layer is lost through cosine similarity convergence (CSC).
These measures are used to evaluate the sensitivity of an attribution method for the later layers of a DNN. 
However, these measures require human judgment to determine whether a test is passed and additional metrics for images.
Instead, we focus on \textit{how much the attributions obtained from GBA methods can improve the classification accuracy of a DNN without retraining}.

Before proposing our approach, we theoretically provide a \textit {no implicit bias} condition that is necessary to improve the accuracy. This condition is derived by compressing a DNN into a single-layer neural network and separating the noise terms.
From the perspective of compressing a DNN, we reinterpret the gradient computed with GBA methods and show the condition of the gradient toward reducing the noise of input features.

Our approach transforms the attributions obtained from a pre-trained DNN into an \textit{attribution mask}, multiply it to the input features, and subsequently introduce these masked features as an input to the DNN.
This mask is designed to recursively focus attention on the parts of the input related to the target label. 
To evaluate this approach, we measure the \textit{masked input accuracy} (MIA), which denotes the accuracy obtained by feeding the masked features to the DNN.
Furthermore, we provide Gradient\,*\,Sign-of-Input (GxSI), a new GBA method to achieve better MIA score than existing methods. 

We measure the MIA of the GBA methods on CIFAR-10 and CIFAR-100. In our experiments, we show the followings:
\begin{itemize}
    \item GBA methods of calculating the gradient that satisfy the \textit{no implicit bias} condition achieve a higher MIA scores compared to other methods.
    \item 
    While the existing GBA methods fail to maintain or increase the MIA score as the number of \textit{iterative attribution masking} increases, the proposed GxSI method does not suffer from this problem.
    For example, on CIFAR-10, we achieve the MIA 99.8\%$\sim$99.9\% by masking each image in both the training and test set.
    \item By employing the masked features, obtained using GxSI and the attribution mask, to train and test another DNN, we achieve the image classification accuracy ranging from 99.7\% to 100.0\% on CIFAR-10.
\end{itemize}
These results imply that our approach effectively filters out parts of the input features not related to the target label.

\section{Compression of DNNs and No Implicit Bias}
Before proposing our approach, we reinterpret the gradient used by the existing GBA methods. 
The nonlinearity is locally linearized and the DNN is compressed into a single-layer neural network (Section \ref{subsec:compressing}). 
To remove the noise in the compressed DNN, we constrain the model to a certain condition (Section \ref{subsec:bf}).
In Section \ref{subsec:ag}, we show that the noise of gradient can be reduced under this condition.
\subsection{Linearization and compression of DNNs}
\label{subsec:compressing}
A single-layer neural network is an easy-to-interpret model because it simplifies the relationship between the target class and the input features to an affine transform. Therefore, to create an interpretable DNN, we compress a DNN into a single-layer neural network through local linearization. 
Suppose that a pre-trained DNN to be compressed has a total of $L$ layers, and the number of dimensions of the $l$-th layer is $d_l$, where $l\in\{1,\dots,L\}$. 
Let $\mathbf{z}_l=\phi(\mathbf{s}_l)$ be a hidden output of each layer. Here, the input of nonlinearity $\phi(\cdot)$ is $\mathbf{s}_l=\mathbf{W}_l\mathbf{z}_{l-1}+\mathbf{b}_l$, the weight and the bias are $\mathbf{W}_l\in R^{d_{l}, d_{l-1}}$ and $\mathbf{b}_l\in R^{d_l}$, and the input features are $\mathbf{z}_0\in R^{d_0}$. 
\begin{figure}[]
	\centering
	\includegraphics[scale=0.31]{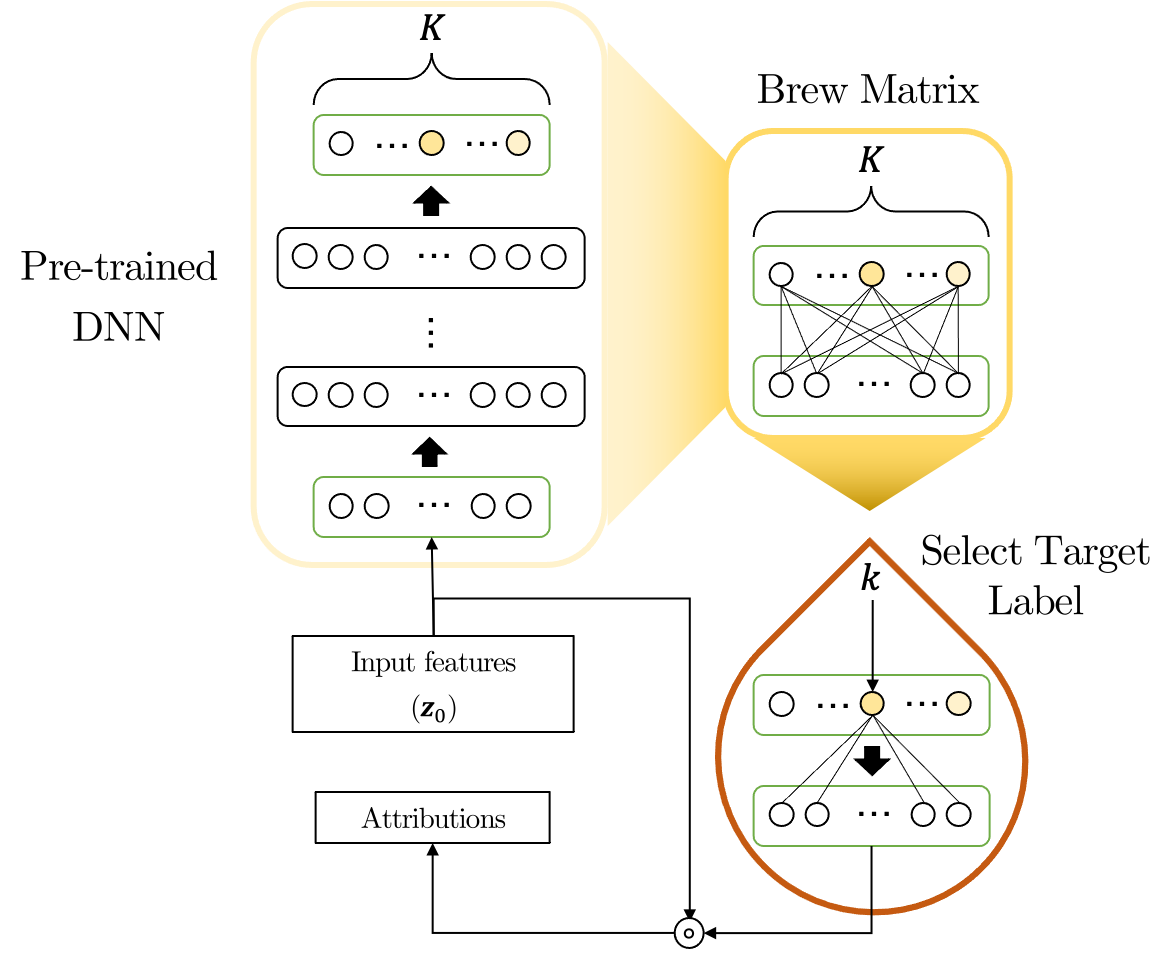}
	\caption{
	A DNN can be compressed into a single-layer neural network through local linearization, whose weight matrix is called a \textit{brew matrix}. The row vector at the target label position $k$ of the brew matrix is used to obtain the attributions with respect to the given input feature.
	}
\end{figure}

To locally linearize a DNN, we focus on the nonlinearity of the DNN. The nonlinearity can be classified into two categories, element-wise and non-element-wise nonlinearity. The element-wise nonlinearity (e.g., Rectified Linear Unit (ReLU), Tanh, etc.) denotes a function that does not change the dimensions of the input features, and the non-element-wise nonlinearity (e.g., MaxPool, AvgPool, etc.) means that the input dimension changes. 
First, applying the Taylor decomposition \cite{kindermans2016investigating, montavon2017explaining} to the element-wise nonlinearity function $\phi: R^{d_l}\to R^{d_l}$ with respect to an input of the function $\mathbf{s}_l\in R^{d_l}$, we obtain
\begin{equation}
    \label{equ:nonelenonlinear_wise}
    \phi(\mathbf{s}_l)=\mathbf{a}_l\circ\mathbf{s}_l+\mathbf{c}_l,
\end{equation}
where $\mathbf{a}_l=\nabla_{\mathbf{s}_l}\phi(\mathbf{s}_l)$ and $\mathbf{c}_l=\phi(\mathbf{e}_l)-\mathbf{a}_l\circ\mathbf{e}_l+(\nabla^2_{\mathbf{s}_l}\phi(\mathbf{s}_l)\circ(\mathbf{s}_l-\mathbf{e}_l)^2+\dots)$. The baseline is $\mathbf{e}_l\in R^{d_l}$ and $\circ$ is an element-wise product.
Next, for a non-element-wise nonlinearity $\phi: R^{d_l}\to R^{d_{l+1}}$, the Taylor decomposition provides
\begin{equation}
    \label{equ:nonelenonlinear}
    \phi(\mathbf{s}_l)=\mathbf{W}_{l+1}\mathbf{s}_l,
\end{equation}
where $\mathbf{W}_{l+1}\in R^{d_{l+1}, d_{l}}$. Because there is no bias in the result of the local linearization, we can absorb the function in another linear layer. Therefore, we only need to consider the element-wise nonlinearity containing bias and apply the linearized nonlinearity to the DNN $f: R^{d_0}\to R^{d_L}$ including CNNs \cite{nie2018theoretical}. We then have
\begin{equation}
    \label{equ:lineardnn}
    \begin{split}
    f(\mathbf{z}_0)= \mathbf{a}_L\circ(\mathbf{W}_L(\dots(\mathbf{a}_1\circ(\mathbf{W}_1\mathbf{z}_0+\mathbf{b}_1)+\mathbf{c}_1) \\
    +\dots)+\mathbf{b}_L)+\mathbf{c}_L.
    \end{split}
\end{equation}

\begin{lemma} [] 
\label{theo:associative}
Let $\mathbf{W}\in R^{d',d}$, $\mathbf{x}\in R^d$ and $\mathbf{a}\in R^{d'}$. Then,
    \[
        \mathbf{a}\circ(\mathbf{W}\mathbf{x})=(\bar{\mathbf{a}}\mathbf{W})\mathbf{x},
    \]
where $\bar{\mathbf{a}}=\text{diag}(\mathbf{a})$.
\end{lemma}
\begin{proof}
\[
    \mathbf{a}\circ(\mathbf{W}\mathbf{x})=\text{diag}(\mathbf{a})(\mathbf{W}\mathbf{x}) \\
    = (\text{diag}(\mathbf{a})\mathbf{W})\mathbf{x}.
\]
\end{proof}
Lemma \ref{theo:associative} is applied to Equation \ref{equ:lineardnn} of the locally linearized DNN for the input features and is derived as follows:
\begin{equation}
    \label{equ:shotlineardnn}
    f(\mathbf{z}_0)=\mathbf{W}_t\mathbf{z}_0+\mathbf{\Delta}_t,
\end{equation}
where
\begin{equation}
    \label{equ:shotlineadnnmatrix}
    \mathbf{W}_t={\prod_{l=1}^{L}(\bar{\mathbf{a}}_l\mathbf{W}_l)},
\end{equation}
\begin{dmath}
\label{equ:addbias}
    \mathbf{\Delta}_t=\sum_{l=1}^{L-1}{\prod_{m=l+1}^{L}(\bar{\mathbf{a}}_m\mathbf{W}_m)(\mathbf{a}_l\circ\mathbf{b}_l+\mathbf{c}_l)}+\mathbf{a}_L\circ\mathbf{b}_L+\mathbf{c}_L.    
\end{dmath}
Assuming there is $\mathbf{\Delta}^*_t\in R^{d_0}$ that satisfies $\mathbf{\Delta}_t=\mathbf{W}_t\mathbf{\Delta}^*_t$ for compressing the DNN into a weight matrix, Equation \ref{equ:shotlineardnn} is derived as follows:
\begin{equation}
    \label{equ:brewmatrix}
    f(\mathbf{z}_0)=\mathbf{W}_t(\mathbf{z}_0+\mathbf{\Delta}_t^*).
\end{equation}
Thus, the brew matrix $\mathbf{W}_t$ and the additive input feature $\mathbf{\Delta}_t^*$ can be obtained for the input features $\mathbf{z}_0$. 
\subsection{No implicit bias model for noise reduction}
\label{subsec:bf}
The additive input features $\mathbf{\Delta}_t^*$ is generated by the total bias $\mathbf{\Delta}_t$, which can work as noise for the input features $\mathbf{z}_0$ and is consistent with what was claimed in \cite{mohan2019robust, wang2019bias}. To remove $\mathbf{\Delta}_t^*$, we use a \textit{no implicit bias} (NIB) model, a condition that forces all $\mathbf{c_l}$ and $\mathbf{b_l}$ in Equation \ref{equ:addbias} to zeros. 

First, to force $\mathbf{b}_l=\mathbf{0}$, we eliminate the biases of all layers as well as the mean parameters of batch normalization \cite{ioffe2015batch}. 
Next, we use the Taylor decomposition under the baselines $\mathbf{e}_l=\mathbf{0}$ to find the nonlinearity such that it satisfies the condition $\mathbf{c}_l=\mathbf{0}$. 
Therefore, we only consider functions such that satisfy this condition (e.g., ReLU or MaxPool).
For example, in the case of ReLU, $\phi(\mathbf{0})=\mathbf{0}$ and $\nabla^{(n)}_{\mathbf{s}_l}\phi(\mathbf{s}_l)=\mathbf{0}$ for $n\geq 2$, and thus Equation \ref{equ:nonelenonlinear_wise} is calculated as $\mathbf{a}_l\circ\mathbf{s}_l$.
In the other case, as in Equation \ref{equ:nonelenonlinear}, MaxPool satisfies the condition because the function is locally linearized as a linear layer without bias.  
These two conditions, $\mathbf{b}_l=\mathbf{0}$ and $\mathbf{c}_l=\mathbf{0}$, lead $\mathbf{\Delta}_t$ to $\mathbf{0}$ in Equation \ref{equ:addbias}. We thus can derive Equation \ref{equ:brewmatrix} as follows: 
\begin{equation}
    f(\mathbf{z}_0)=\mathbf{W}_t\mathbf{z}_0.    
\end{equation}
Consequently, $\mathbf{W}_t\in R^{d_0,d_L}$ denotes the noise-reduced brew matrix. When using a DNN as a classifier, which means that $d_L$ is the same as the number of classes $K$, the simplified matrix fully contains the information regarding the target class $k\in\{1,2,\dots,K\}$. Therefore, because the matrix is calculated using only the input features excluding $\mathbf{\Delta}_t^*$, we simply use the $k$-th row of the matrix to obtain information about the target label.

In terms of the brew matrix, attributions $[\mathbf{W}_t]_{k}\circ\mathbf{z}_0$ satisfy the efficiency properties of Shapley when the attributions are defined as Shapley values
\cite{sundararajan2020many}. In other words, the sum of all elements is equivalent to the prediction for the $k$-th class of the DNN.
\subsection{No implicit bias condition}
\label{subsec:ag}
Existing GBA methods calculate attributions along with input features by modifying the gradient obtained from a DNN.
Among such methods, Gradient\,*\,Input, which is the basis, obtains attributions by applying the gradient without additional manipulation. Given the $L$-layer DNN $f$ and the input features $\mathbf{z}_0$, the gradient of this method is calculated through an automatic differentiation as follows:
\begin{dmath}
    \label{equ:aglienardnnn}
    \begin{split}
    \nabla_{\mathbf{z}_0}f(\mathbf{z}_0)=\prod_{l=1}^{L}
    (\nabla_{\mathbf{s}_{l}}\mathbf{z}_{l}\nabla_{\mathbf{z}_{l-1}}\mathbf{s}_{l}).
    \end{split}
\end{dmath}
Here, $\nabla_{\mathbf{s}_{l}}\mathbf{z}_{l}$ is a locally linearized nonlinearity $\bar{\mathbf{a}}$ and $\nabla_{\mathbf{z}_{l-1}}\mathbf{s}_l$ is a weight matrix $\mathbf{W}_l$. 
Therefore, Equation \ref{equ:aglienardnnn} is identical to Equation \ref{equ:shotlineadnnmatrix} under two conditions, $\mathbf{c}_l=0$ and $\mathbf{b}_l=0$. Thus, the brew matrix can also be computed as follows:
\begin{equation}
    \mathbf{W}_t=\nabla_{\mathbf{z}_0}f(\mathbf{z}_0).    
\end{equation}
Based on this result, the gradient calculated using the GBA methods ignores the additive noise term, which means that this noise is always included in the attributions when the NIB model is not used.
\begin{definition}
Given a DNN $f:R^{d_0}\to R^{d_L}$, the input features are $\mathbf{z}_0$ and function $g:R^{d_L}\to R^{d_L,d_0}$ modifies the gradient of this model. For the baselines $\mathbf{e}_0$ that satisfy $f(\mathbf{e}_0)=0$, if
        $f(\mathbf{z}_0)=g(f, \mathbf{z}_0)(\mathbf{z}_0-\mathbf{e}_0)$,
we state that GBA method satisfies the NIB condition.
\end{definition}
The NIB condition guarantees that the gradient does not contain additive noise. DeepLIFT and Integrated Gradients, which are designed for \textit{Completeness} \cite{sundararajan2017axiomatic}, satisfy this condition by introducing the baselines, although finding such baselines is a difficult problem. Therefore, we compute the gradient under the NIB model to enhance the model accuracy through attributions. Under this model, Gradient\,*\,Input, DeepLIFT and Integrated Gradients satisfy the NIB condition (see Appendix A).

\section{Attribution Mask and Proposed Attribution Method}
\label{subsec:proposed}
We now describe our approach for measuring and improving the accuracy of DNNs with attributions. To achieve this, we propose an attribution mask that is recursively calculated from the GBA method (Section \ref{subsec:attrm}). We also provide a new attribution method for improving the performance (Section \ref{subsec:gsi}).
\subsection{Attribution mask and the masked input accuracy (MIA)}
\label{subsec:attrm}
The gradient, which is denoted as a specific row of the brew matrix, is attributed to the class $k$. Each row is trained using several input features corresponding to each class within the dataset. 
Because these features contain some redundant information other than that related to the target class, we suppress this information by focusing attention on important parts of input features through the attributions.
\begin{algorithm}
    \caption{Method for obtaining the attribution mask}
    \label{algo:measure}
    \begin{algorithmic} 
        \STATE \textbf{Input}: input features $\mathbf{z}_0$ 
        \STATE \textbf{Target}: target class label $k$ \\
        \STATE \textbf{Model}: pre-trained DNN model $f$ \\
        \STATE \textbf{Initialization}: initial attribution mask $\mathbf{m}^{*}_{0,k}\leftarrow\mathbf{1}$ \\ 
        \FOR{$i=1,\dots,I$}
            \STATE $\mathbf{A}_{i,k}= \text{Expl}(f,\mathbf{z}_0\circ\mathbf{m}^*_{i-1, k}, k)$ \\
            \STATE $\mathbf{m}_{i,k} = \text{MinMax}(\mathbf{A}_{i,k}$) \\
            \STATE $\mathbf{m}^*_{i,k}\leftarrow\mathbf{m}^*_{i,k}\circ\mathbf{m}_{i,k}$ \\
        \ENDFOR
        \STATE \textbf{Output}: $\mathbf{m}^*_{i,k}$
    \end{algorithmic}
\end{algorithm}

First, we propose an attribution mask that recursively focuses attention on parts of an input of a pre-trained DNN $f$. Let attributions $\mathbf{A}_{i,k}=\text{Expl}(f,\mathbf{z}_0\circ\prod_{j=0}^{i-1}\mathbf{m}_{j,k}, k)$ with $\text{Expl}$ denotes the GBA method, where the attribution mask for the target class is $k$ and the number of iterations is $i$.
The mask is calculated using $\text{MinMax}$ normalization \cite{patro2015normalization} as follows:
\begin{equation}
    \mathbf{m}_{i,k}=\dfrac{\mathbf{A}_{i,k}-\min(\mathbf{A}_{i,k})}{\max(\mathbf{A}_{i,k})-\min(\mathbf{A}_{i,k})},
\end{equation}
where the initial mask $\mathbf{m}_{0,k}$ is $\mathbf{1}$, and $\mathbf{m}_{i,k}, \mathbf{A}_i\in R^{d_0}$. This method normalizes the mask value to between zero and 1, leaving only the ratio of the relative values between the attributions, and thus the repeated multiplications of the mask does not diverge.
In the case of an image, which is the target data of this study, $d_0$ is $channel \times height \times width$, and we conduct normalization for each $channel$. 

Next, the attribution mask is multiplied by the input features and then the masked features are fed into the DNN. 
Let $\mathcal{D}$ be a test dataset that has the input features $\mathbf{z}_0$ paired with the target label $t\in\{1,2,\dots,K\}$.
Then, the accuracy of the DNN with the mask is calculated as follows:
\begin{equation}
    Q = \dfrac{1}{|\mathcal{D}|}\sum_{\{\mathbf{z}_0,t\}\in\mathcal{D}}\mathbbm{1}_{t}[\argmax_{k}f(\mathbf{z}_0\circ\prod_{i}^{}\mathbf{m}_{i,t})],
\end{equation}
where $\mathbbm{1}_{t}[.]$ is an indicator function, which is $1$ when $t=k$ and has a value of $0$ for the others. 
This accuracy, i.e., the MIA, is used to quantify the loss of information that occurs when target label information passes through DNN layers by the backpropagation process.
For example, if $Q$ decreases from the baseline accuracy ($i=0$) for $i>0$, the target label information is not properly backpropagated to the attributions, indicating that the irrelevant part is incorrectly removed.

\subsection{Gradient\,*\,Sign-of-Input (GxSI)}
\label{subsec:gsi}
The existing GBA methods use the gradient that already has the physical dimension of the input features and create nonlinear attributions under the perspective of the brew matrix. To make linear attributions, we only use the sign of the input features. 
Next, we propose the attribution method, which is calculated as follows:
\begin{equation}
    \label{equ:gsi}
    \mathbf{A}_{i,k}=\text{sign}(\mathbf{z}_0)\circ([\nabla_{\mathbf{z}_0}f(\mathbf{z}_0)]_{k}),
\end{equation}
We label this method Gradient\,*\,Sign-of-Input (GxSI), which increases the attribution value when the gradient element and the input feature have the same sign. 
This method is based on the intuition that the input features that improve the prediction of the DNN are an important part.
Unlike the Gradient\,*\,Input, this method is designed so that the gradient value affects attributions by only using the sign of the input instead of its value.

\section{Experiment Setup}
We tested our approach on two image datasets CIFAR-10 and CIFAR-100. For these datasets, the mean-variance normalization was used. This normalization was conducted for each R, G and B channel, the means of which were 0.507, 0.487, and 0.441, and the standard deviations were 0.267, 0.256, and 0.276, respectively. 
We trained the DNN based on PyTorch \cite{paszke2019pytorch} and used the open-source Captum \cite{kokhlikyan2020captum} based on PyTorch to test the existing GBA methods.
To improve the performance of the DNN, we used a random crop and random horizontal flip on the images during training. Weights are randomly initialized prior to training.
The optimizer used for training the DNN applied stochastic gradient descendent (SGD), and the momentum was 0.9 and the weight decay is 0.0005. 
By allocating 88 batches for each of the 4 GPUs, the total batch size was 352. 
The learning rate starts with 0.1 and uses a learning rate decay method that reduces the learning rate in a specific epoch. 
The specific epoch depends on the dataset, and for CIFAR-10, the learning rate was multiplied by 0.1 in 60 out of a total of 70 epochs. For CIFAR-100, the learning rate was multiplied by 0.1 in 100 and 150 out of a total of 175 epochs.

\subsection{Model}
\label{subsec:setup}
\textbf{NIB VGG}
To construct the NIB model, we removed all parameters corresponding to the mean of batch normalization in VGG \cite{simonyan2014very} and the additive bias of all layers. In this model, a kernel size of 3 and a padding of 1 were used for all conv2d layers, and a kernel size of 2 and a stride size of 2 were used for MaxPool.
To experiment with NIB models of different depths, we started with NIB VGG13 and constructed the model by reducing the number of conv2d layer by 2 to a total of 7 layers (see Table \ref{tab:spec}). Their models are labeled NIB VGG13-2, NIB VGG13-4, and NIB VGG13-6, respectively.
\begin{table}[]
    \caption{Configurations of NIB VGGs. For example, ``layer-512" means the layer has 512 output nodes.}
    \label{tab:spec}
    \centering
    \resizebox{0.8\columnwidth}{!}{
        \begin{tabular}{|c|c|c|c|c|}
        \hline
        NIB VGG13-6 & NIB VGG13-4  & NIB VGG13-2       & NIB VGG13                                                                          & NIB VGG16                                                                                  \\ \hline\hline
        \multicolumn{5}{|c|}{\begin{tabular}[c]{@{}c@{}}conv2d-64\\ conv2d-64\\ maxpool\end{tabular}}                                                                                                                \\ \hline
        \multicolumn{5}{|c|}{\begin{tabular}[c]{@{}c@{}}conv2d-128\\ conv2d-128\\ maxpool\end{tabular}}                                                                                                              \\ \hline
        -       & \multicolumn{3}{c|}{\begin{tabular}[c]{@{}c@{}}conv2d-256\\ conv2d-256\\ maxpool\end{tabular}}            & \begin{tabular}[c]{@{}c@{}}conv2d-256\\ conv2d-256\\ conv2d-256\\ maxpool\end{tabular} \\ \hline
        -       & -        & \multicolumn{2}{c|}{\begin{tabular}[c]{@{}c@{}}conv2d-512\\ conv2d-512\\ maxpool\end{tabular}} & \begin{tabular}[c]{@{}c@{}}conv2d-512\\ conv2d-512\\ conv2d-512\\ maxpool\end{tabular} \\ \hline
        -       & -        & -             & \begin{tabular}[c]{@{}c@{}}conv2d-512\\ conv2d-512\\ maxpool\end{tabular}      & \begin{tabular}[c]{@{}c@{}}conv2d-512\\ conv2d-512\\ conv2d-512\\ maxpool\end{tabular} \\ \hline
        \multicolumn{5}{|c|}{FC-4096}                                                                                                                                                                                \\ \hline
        \multicolumn{5}{|c|}{FC-4096}                                                                                                                                                                                \\ \hline
        \multicolumn{5}{|c|}{FC-10 or FC-100}                                                                                                                                                                                  \\ \hline
        \end{tabular}
    }
\end{table}

\textbf{WideResNet}
WideResNet \cite{zagoruyko2016wide} was only used as a student model in Section \ref{subsec:generalmask}, and thus it was not modified as an NIB model. We used WRN-28-10 and WRN-40-4, both of which achieve a stable performance on CIFAR datasets.

\subsection{Attribution methods}
\label{subsec:attmethodsettings}
We targeted the following GBA methods: Gradient\,*\,Input (GxI),  Guide Backpropagation (GBP), Integrated Gradient (IG), DeepLIFT (DL), Positive-Gradient\,*\,Input (PGxI) and Gradient\,x\,Sign-of-Input (GxSI). 
Our implementations of GBA method are identical to those of \cite{adebayo2018sanity}, except for PGxI and GxSI.
First, in the case of the proposed attribution method, GxSI, we implemented Equation \ref{equ:gsi}.
Next, in the case of PGxI, we implemented the saliency map of \cite{simonyan2013deep} by multiplying the input by the following element: $[\nabla_{\mathbf{z}_0}f(\mathbf{z}_0)]_{k}^{+}\circ\mathbf{z}_0$, where $k$ is a target class and $[\cdot]^{+}=\max(\cdot, 0)$.
Among the other methods, IG and DL need the baselines, and because we use the NIB model, the attributions were calculated with the baselines set to zeros. In particular, in the case of IG, we calculated the attributions with a step size of 50 to numerically calculate the integral.

\section{Evaluation and Discussion}
\begin{figure}
	\includegraphics[scale=0.5]{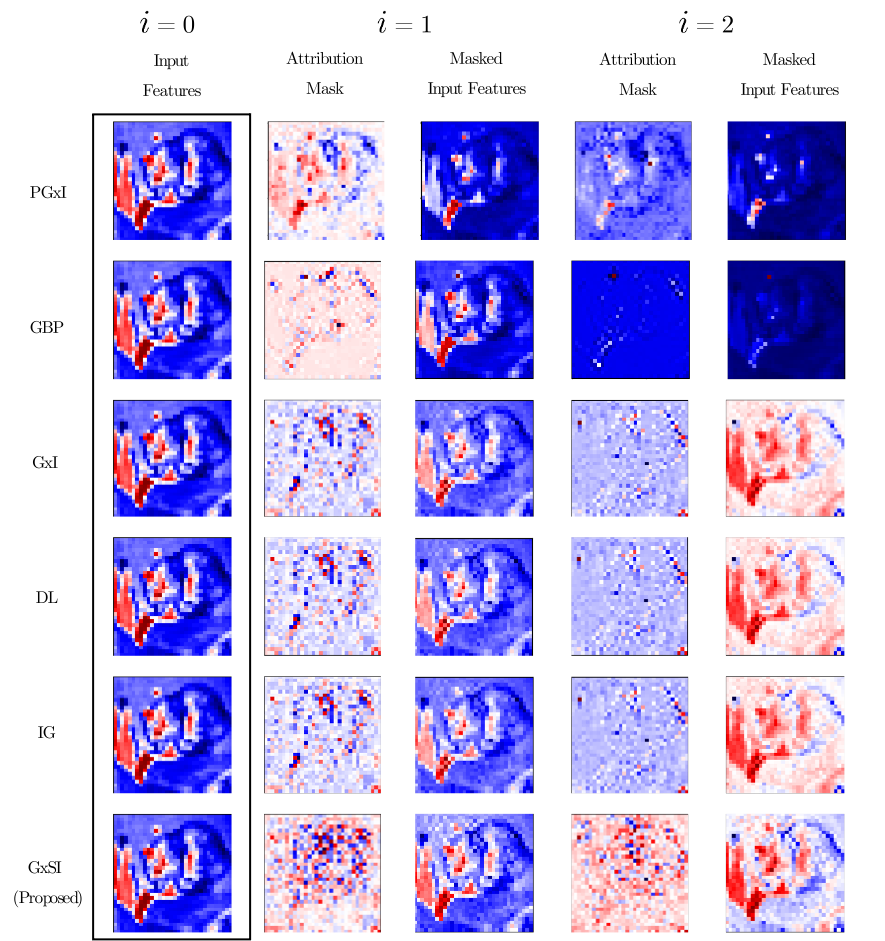}
	\caption{Comparison of different attribution methods. Input features are fed to the DNN without the attribution mask at iteration $i=0$. After calculating the mask, it is multiplied to the input features as the iterations increase. Red color denotes positive and blue color is negative relevance.}
	\label{fig:analysis}
\end{figure}
We compare the performance of the GBA methods through the MIA (Section \ref{subsec:analysis}). 
We test the effectiveness of GxSI (Section \ref{subsec:analysis_gsi} and Section \ref{subsec:interprep_model}).
We modify a dataset with the attribution masks obtained by GxSI and subsequently use this dataset to train another DNN model (Section \ref{subsec:generalmask}).
Some ablation studies are provided in Section \ref{subsec:ablestudy}.

\subsection{MIA score of GBA methods}
\label{subsec:analysis}
To verify the effect of the NIB condition on the MIA score, we measured this score of the existing GBA methods. Whether the methods satisfy the condition or not depends on the gradient modification and a pre-trained model. We used NIB VGG16 among the pre-trained NIB models to demonstrate the effect of the gradient modification (see the details in Section \ref{subsec:setup}). 

As shown in Table \ref{tab:interpret_iteration}, GxI, DL, and IG have the same MIA value under the NIB model. We also see that these methods produce better performance than other methods. 
In particular, on CIFAR-10, the performance was 96.3\%, which is close to 100\%. 
However, as shown in Figure \ref{fig:analysis}, PGxI and GBP methods reduce the performance despite the attribution mask looking more similar to the input image than with the other approaches. 
This deterioration in performance is because the object in this image disappears as the mask recursively modifies the input image (Figure \ref{fig:analysis} col. 1, col. 2).
The baseline accuracy (i.e., the accuracy at $i=0$) of the DNN on CIFAR-100 is inferior to that of CIFAR-10. However, we can see that the methods, which satisfy the NIB condition, have a higher performance than the other methods.

According to the above results, methods using gradients that satisfy the NIB condition can increase the MIA compared to the other methods. Meanwhile, these results are contrary to the fact that GxI and other GBA methods failed the sanity check \cite{adebayo2018sanity, sixt2020explanations}, the reason for which is that these methods do not satisfy the NIB condition because the model used in the previous study was not the NIB model.

\subsection{Improvement of the MIA through GxSI}
\label{subsec:analysis_gsi}
\begin{table}[]
    \caption{The MIA $Q$ comparison of each attribution methods for NIB VGG16 on CIFAR-10 and CIFAR-100. $Q\uparrow$ denotes difference $Q$ from the baseline accuracy ($i=0$). Iteration denote number of applying the attribution mask to the input features.}
    \centering
	\resizebox{\columnwidth}{!}{
	    \label{tab:interpret_iteration}
        \begin{tabular}{|c|c|c|c|c|c|c|}
        \hline
        \multirow{3}{*}{Dataset}   & \multirow{3}{*}{Method}                                   & \multicolumn{5}{c|}{Iteration}                                                        \\ \cline{3-7} 
                                   &                                                           & \multirow{2}{*}{$i=0$} & \multicolumn{2}{c|}{$i=1$}    & \multicolumn{2}{c|}{$i=2$}   \\ \cline{4-7} 
                                   &                                                           &                        & $Q$ (\%)      & $Q\uparrow$   & $Q$ (\%)      & $Q\uparrow$  \\ \hline\hline
        \multirow{6}{*}{CIFAR-10}  & PGxI                                                      & \multirow{6}{*}{91.5}  & 80.6          & -10.5         & 64.5          & -26.0        \\ \cline{2-2} \cline{4-7} 
                                   & GBP                                                       &                        & 85.8          & -5.7          & 69.8          & -21.7        \\ \cline{2-2} \cline{4-7} 
                                   & GxI                                                       &                        & 90.2          & -1.3          & 96.3          & 4.8          \\ \cline{2-2} \cline{4-7} 
                                   & DL                                                        &                        & 90.2          & -1.3          & 96.3          & 4.8          \\ \cline{2-2} \cline{4-7} 
                                   & IG                                                        &                        & 90.2          & -1.3          & 96.3          & 4.8          \\ \cline{2-2} \cline{4-7} 
                                   & \begin{tabular}[c]{@{}c@{}}GxSI\\ (Proposed)\end{tabular} &                        & \textbf{94.1} & \textbf{2.6}  & \textbf{98.2} & \textbf{6.7} \\ \hline
        \multirow{6}{*}{CIFAR-100} & PGxI                                                      & \multirow{6}{*}{69.8}  & 50.5          & -19.3         & 30.1          & -39.7        \\ \cline{2-2} \cline{4-7} 
                                   & GBP                                                       &                        & 61.4          & -8.4          & 38.2          & -31.6        \\ \cline{2-2} \cline{4-7} 
                                   & GxI                                                        &                        & 57.5          & -12.3         & 64.1          & -5.7         \\ \cline{2-2} \cline{4-7} 
                                   & DL                                                        &                        & 57.5          & -12.3         & 64.1          & -5.7         \\ \cline{2-2} \cline{4-7} 
                                   & IG                                                        &                        & 57.5          & -12.3         & 64.1          & -5.7         \\ \cline{2-2} \cline{4-7} 
                                   & \begin{tabular}[c]{@{}c@{}}GxSI\\ (Proposed)\end{tabular} &                        & \textbf{66.2} & \textbf{-3.6} & \textbf{75.9} & \textbf{6.1} \\ \hline
        \end{tabular}
    }
\end{table}
\begin{figure*}[ht!]
	\centering
	\includegraphics[scale=0.70]{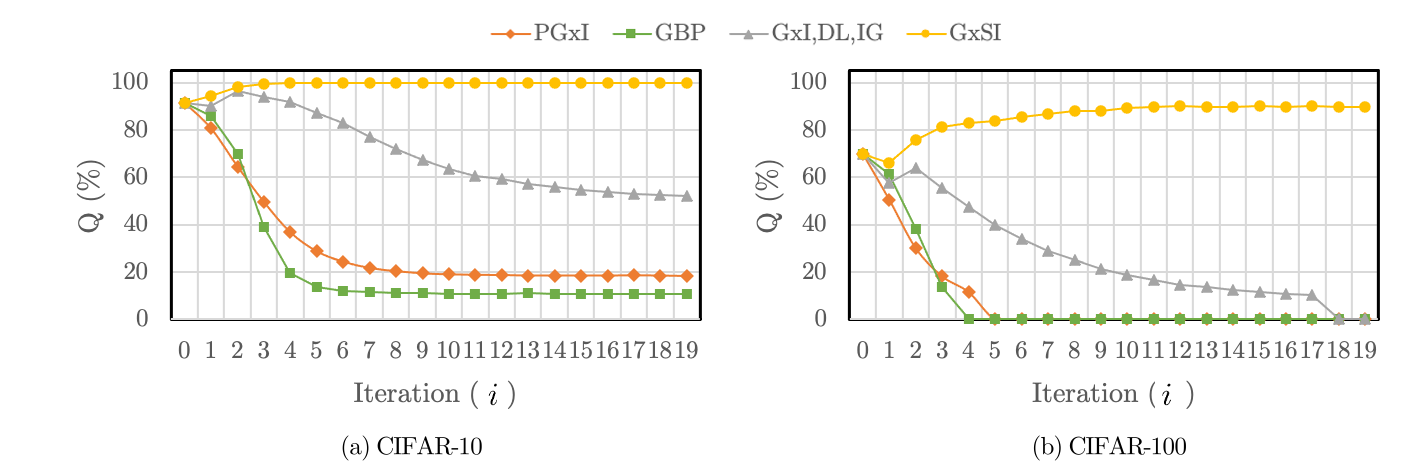}
	\caption{The MIA score for NIB VGG16 model as the iteration increases. (a) The MIA score of each attribution method on (a) CIFAR-10 and (b) CIFAR-100. Note that GxI, DL, and IG share the same legend entry (gray color), as their scores hardly differ from each other.}
	\label{fig:interpret_iteration}
\end{figure*}
To validate that GxSI improves the performance, we compare its MIA score with that of the methods discussed in Section \ref{subsec:analysis}. As we can see in Table \ref{tab:interpret_iteration}, the proposed method achieves superior performance compared to the other attribution methods, and the gradient satisfies the property. This superiority becomes more pronounced when the number of iterations increase to 3 or more.

In Figure \ref{fig:interpret_iteration} (a), we can see that the proposed method, unlike the other approaches, continuously increased its MIA and reached a performance close to the ideal value of 100\%. 
In particular, on CIFAR-10, the performance reached 99.8\% with 7 iterations, 99.9\% with 10 iterations, and then fluctuated between 99.8\% and 99.9\%.
Meanwhile, for GxI, IG, and DL, the measured value increased only until the second iteration and then gradually decreased to 52\%.
Moreover, for PGxI and GBP, the performance decreased continuously to 18.3\% and 10.9\%, respectively.
In Figure \ref{fig:interpret_iteration} (b), on CIFAR-100, we can also see that the MIA value measured using the GxSI increases continuously, and we achieved 90.2\% of the MIA with 12 iterations despite the relatively low baseline accuracy (see Appendix B).

These results imply that the attribution mask of the GxSI can recursively filter information not related to the target label more effectively than existing methods.
\subsection{MIA of GxSI and DNN configuration}
\label{subsec:interprep_model}
\begin{table}[]
    \caption{The MIA of GxSI comparison of each model architecture on CIFAR-10. The specific model structure corresponding to each model name can be checked in Table \ref{tab:spec}. ``\# of layer" denotes the number of convolutional layers and linear layers.}
    \label{tab:MIAofModel}
    \centering
	\resizebox{\columnwidth}{!}{
        \begin{tabular}{|c|c|c|c|c|c|c|}
        \hline
        \multirow{3}{*}{Name} & \multirow{3}{*}{\# of layer} & \multicolumn{5}{c|}{Iteration}                                                       \\ \cline{3-7} 
                              &                              & \multirow{2}{*}{$i=0$} & \multicolumn{2}{c|}{$i=1$}      & \multicolumn{2}{c|}{$i=2$}      \\ \cline{4-7} 
                              &                              &                      & $Q (\%)$      & $Q\uparrow$   & $Q (\%)$      & $Q\uparrow$   \\ \hline\hline
        NIB VGG13-6               & 7                            & 83.3                 & \textbf{98.4} & \textbf{15.1} & \textbf{99.9} & \textbf{16.6} \\ \hline
        NIB VGG13-4               & 9                            & 89.0                 & 97.5          & 8.5           & 99.7          & 10.7          \\ \hline
        NIB VGG13-2               & 11                           & 90.9                 & 97.2          & 6.3           & 99.6          & 8.7           \\ \hline
        NIB VGG13                 & 13                           & 91.4                 & 96.5          & 5.1           & 99.1          & 7.7           \\ \hline
        NIB VGG16                 & 16                           & \textbf{91.5}        & 94.1          & 2.6           & 98.2          & 6.7           \\ \hline
        \end{tabular}
    }
\end{table}
\begin{table}[]
    \caption{Student-teacher model accuracy under different model combination on CIFAR-10. Usage of attribution mask is denoted by `w/' and vice versa. Student accuracy `w/o' mask is the baseline accuracy of each model architecture.}
    \label{tab:general}
    \centering
	\resizebox{\columnwidth}{!}{
        \begin{tabular}{|c|c|c|c|c|}
        \hline
        \multirow{2}{*}{\begin{tabular}[c]{@{}c@{}}Student\\ Model\end{tabular}} & \multirow{2}{*}{\begin{tabular}[c]{@{}c@{}}Student\\ Acc.\\ w/o Teacher\end{tabular}} & \multicolumn{3}{c|}{\begin{tabular}[c]{@{}c@{}}Student\\ Acc.\\ w/ Teacher\end{tabular}}                                                                                                                           \\ \cline{3-5} 
                                                                                 &                                                                                       & \begin{tabular}[c]{@{}c@{}}NIB VGG13-6\\ ($Q$=99.9,\\ $i$=3)\end{tabular} & \begin{tabular}[c]{@{}c@{}}NIB VGG16\\ ($Q$=98.2,\\ $i$=3)\end{tabular} & \begin{tabular}[c]{@{}c@{}}NIB VGG16\\ ($Q$=99.8,\\ $i$=10)\end{tabular} \\ \hline\hline
        NIB VGG13-6                                                                  & 83.3                                                                                  & 99.7                                                                  & 99.1                                                                & 99.5                                                                 \\ \hline
        NIB VGG13-4                                                                  & 89.0                                                                                  & 99.8                                                                  & 99.6                                                                & 99.8                                                                 \\ \hline
        NIB VGG13-2                                                                  & 90.9                                                                                  & 99.9                                                                  & 99.7                                                                & 99.9                                                                 \\ \hline
        NIB VGG13                                                                    & 91.4                                                                                  & 99.9                                                                  & 99.7                                                                & 99.9                                                                 \\ \hline
        NIB VGG16                                                                    & 91.5                                                                                  & 99.9                                                                  & 99.6                                                                & 99.9                                                                 \\ \hline
        WRN-40-4                                                                 & 94.3                                                                                  & 100.0                                                                   & 99.8                                                                & 100.0                                                                  \\ \hline
        WRN-28-10                                                                & 95.1                                                                                  & 100.0                                                                   & 99.8                                                                & 100.0                                                                  \\ \hline
        \end{tabular}
    }
\end{table}
\begin{table*}[ht!]
    \caption{
    Comparison of the MIA scores of model that do not satisfy NIB condition on CIFAR-10. BNB denotes the batch normalization bias and LB denotes the layer bias of the whole model.}
    \label{tab:NIBtest}
    \centering
	\resizebox{0.80\textwidth}{!}{
        \begin{tabular}{|c|c|c|c|c|c|c|c|c|c|c|c|}
        \hline
        \multirow{3}{*}{Model}          & \multirow{3}{*}{Method} & \multicolumn{10}{c|}{Iteration}                                                                                        \\ \cline{3-12} 
                                        &                         & \multirow{2}{*}{$i=0$} & $i=1$         & $i=2$         & $i=3$ & $i=4$ & $i=5$ & $i=6$ & $i=7$         & $i=8$ & $i=9$ \\ \cline{4-12} 
                                        &                         &                        & \multicolumn{9}{c|}{$Q$ (\%)}                                                                 \\ \hline\hline
        \multirow{2}{*}{NIB VGG16 + BNB +LB} & GxI                     & \multirow{2}{*}{92.6}  & 77.2          & \textbf{83.1} & 67.3  & 37.6  & 20.7  & 16.0  & 15.0          & 14.8  & 14.7  \\ \cline{2-2} \cline{4-12} 
                                        & GxSI                    &                        & 78.1          & \textbf{80.0} & 66.5  & 40.0  & 22.0  & 13.5  & 11.2          & 10.6  & 10.5  \\ \hline
        \multirow{2}{*}{NIB VGG16 + BNB}     & GxI                     & \multirow{2}{*}{92.9}  & 75.4          & \textbf{83.0} & 73.6  & 45.9  & 24.3  & 18.3  & 17.0          & 16.7  & 16.6  \\ \cline{2-2} \cline{4-12} 
                                        & GxSI                    &                        & 76.0          & \textbf{80.2} & 71.4  & 47.1  & 25.5  & 15.6  & 12.7          & 11.8  & 11.5  \\ \hline
        \multirow{2}{*}{NIB VGG16 + LB}     & GxI                     & \multirow{2}{*}{92.0}  & 89.4          & \textbf{91.1} & 72.2  & 26.9  & 11.7  & 10.6  & 10.5          & 10.5  & 10.4  \\ \cline{2-2} \cline{4-12} 
                                        & GxSI                    &                        & \textbf{92.2} & 91.4          & 75.8  & 36.8  & 14.4  & 10.3  & 10.1          & 10.0  & 10.0  \\ \hline
        \multirow{2}{*}{NIB VGG16}          & GxI                     & \multirow{2}{*}{91.5}  & 90.2          & \textbf{96.3} & 93.9  & 91.7  & 87.0  & 82.8  & 76.8          & 72.0  & 67.5  \\ \cline{2-2} \cline{4-12} 
                                        & GxSI                    &                        & 94.0          & 98.2          & 99.3  & 99.6  & 99.7  & 99.7  & \textbf{99.8} & 99.8  & 99.8  \\ \hline
        \end{tabular}
    }
\end{table*}
To measure the performance of GxSI according to the configurations of a DNN, we constructed the NIB model of various depths (see Table \ref{tab:spec}).

As shown in Table \ref{tab:MIAofModel}, in general, the depth of the DNN and the baseline accuracy are proportional, but the performance of the proposed method was rather reduced. In particular, at $i=2$, the performance of the proposed method for NIB VGG13-6 was 1.7\% higher than that for NIB VGG16. 
However, as shown in Figure \ref{fig:interpret_iteration}, the performance of GxSI for NIB VGG16 increased to 99.9 \%. 

Therefore, these results indicates that the convergence speed of performance decreases according to the model depth, and the deeper DNNs lead to a difficulty in the attributions required to obtain the information of the target label.
However, we can overcome this problem by applying the attribution mask to the input feature recursively.

\subsection{Easy-to-understand features from the attribution mask}
\label{subsec:generalmask}
We modified the input features with the attribution mask obtained using GxSI under the NIB model and applied it for training other models. 
In this way, we tested for arbitrary DNNs in which the mask transforms the input features into easy-to-understand features. 
For this experiment, the model that obtains the attribution mask is called the teacher model, and the model to be trained from the input features modified with this mask is called the student model.
For this teacher model, NIB VGG13 and NIB VGG16 were used to use the attribution mask showing the high MIA values.
Specifically, for NIB VGG16, to introduce an attribution mask showing various MIA scores, this mask was obtained by repeating the process 3 and 10 times. 
The CIFAR-10 dataset was modified by multiplying the mask to the image of the training and test set.
The student model NIB VGGs and WRNs were trained using this dataset.

As shown in Table \ref{tab:general}, the accuracy of all student models was increased compared to the original accuracy. 
For the student model NIB VGG16, the accuracy was 99.9\%, 99.6\%, and 99.8\% with respect to each teacher model, whereas it was 99.7\%, 99.1\%, and 99.5\% for the student model NIB VGG13-6. We can see that the difference in classification accuracy between the two student models was not large.
For the student model WRN-40-4 and WRN-28-10, through the teacher model with an MIA score of 99.8\% and 99.9\%, 100\% accuracy was achieved when the input features contained all of the target label information.
As we can see in the last two columns of Table \ref{tab:general}, the accuracy of $i=10$ NIB VGG16 was higher on average than $i=3$ NIB VGG16. 

These results indicate that the MIA score is a value that generally represents the ability to filter information other than the target label of the attribution mask from the input features. 
If the mask is only effective for the DNN from which it was obtained, the masked input features should depends on the model and lose the generality.
Thus, the masked input features, which allow to achieve a high accuracy, are generally easy-to-understand features for an arbitrary DNN.
Furthermore, we can see that the configuration of the model does not need to be complicated when the arbitrary dataset is masked with the attribution mask.

\subsection{Ablation study}
\label{subsec:ablestudy}
In the previous Sections, we conducted experiments using the NIB model and saw that the NIB condition is necessary for improving the MIA score. 
We revoked this condition and then conducted an experiment to measure the change in the MIA score on CIFAR-10. 
Accordingly, layer bias (LB) and batch normalization bias (BNB) were added to NIB VGG16 for GxI and GxSI among GBA methods, and we measured this score for 10 iterations and compared the maximum MIA score with the baseline accuracy.

As described in Table \ref{tab:NIBtest}, there were two problems, one was that the MIA was small compared to the baseline accuracy, and the other was that the MIA decreased as the iterations increased.
For the model with both LB and BNB added, the MIA scores of GxI and GxSI were decreased to 10.5\% and 12.5\%, respectively, and the maximum MIA was also lower than the baseline accuracy.
The same phenomenon was observed for the model with BNB added, and the two GBA methods showed 16.6\% and 11.5\% MIA scores at 9th iteration, respectively.
For the model with LB added, the maximum MIA of GxSI was slightly increased, but the score eventually fell to 10.0.

This result empirically shows that when the gradient calculated by the GBA methods does not satisfy the NIB condition, the performance is reduced by additive noise as we expected.

\section{Conclusion}
In this study, we proposed the attribution mask that filters out irrelevant parts of the input features and introduced it into the input of a pre-trained DNN to measure the MIA.
For improving the performance of filtering, quantified through the MIA, we also provide the NIB condition and a new attribution method, GxSI.
Through extensive experiments, we measured the performance of GBA methods. We found that the attribution methods, which satisfy the NIB condition, outperformed the methods that did not. 
Remarkably, GxSI showed superior MIA scores compared to the existing methods. We also modified each image in both training and test set with the mask obtained by this method. With this modified dataset, we trained a new DNN from scratch and achieved a classification accuracy of 100\% on the test set.
These results imply that our approach effectively removes parts of input features not related to the target label and is capable of generating easy-to-understand features from the perspective of the DNN.




\nocite{watanabe2018espnet}
\nocite{burgess2019monet}
\nocite{hinton2007learning}
\nocite{lakkaraju2020robust}

\bibliography{attribution_mask}
\bibliographystyle{icml2021}

\onecolumn

\appendix
\section{GBA methods that satisfy the NIB condition under the NIB model}
We show the reason why some GBA methods are satisfied with the NIB condition.
A DNN $f:R^{d}_0\to R^{d}_K$ takes the input features $\mathbf{z}_0\in R^{d}_0$ to predict the target label $k$. 
The model is used as a classifier, and the dimension of the predictions, $d_L$, is equivalent to the number of the target classes $K$.  

\textbf{Gradient\,*\,Input} takes the modified gradient function $g$ as 
\[
g(f, \mathbf{z}_0)=\nabla_{\mathbf{z}_0}f(\mathbf{z}_0).
\]
For the NIB model, $f(\mathbf{z}_0)$ is equivalent to $\nabla_{\mathbf{z}_0}f(\mathbf{z}_0)\mathbf{z}_0$ as shown in Equation 8. Here, the  baselines $\mathbf{e}_0$ are zeros and $f(\mathbf{e}_0)=0$.
Thus, GxI satisfies the NIB condition.

\textbf{DeepLIFT} computed with $f(\mathbf{e}_0)=0$ is equivalent to Gradient\,*\,Input when applied to a model with only ReLU as a nonlinear function and no additive biases (Ancona et al., 2017). 
This model is kind of the NIB model for which Gradient\,*\,Input satisfy the NIB condition.  
Thus, DeepLIFT also satisfies the condition.

\textbf{Integrated Gradients} integrate the gradient calculated with respect to the input features. This integration leads the modified gradient function $g$ as follow:
\[
g(f, \mathbf{z}_0)=\int_{\mathbf{0}}^{\mathbf{1}}\nabla_{\mathbf{z}_0}f(\mathbf{e}_0+\alpha(\mathbf{z}_0-\mathbf{e}_0))\,d\alpha
\]
For the NIB model, $f(\mathbf{e}_0)$ is zero when $\mathbf{e}_0=\mathbf{0}$ because the model does not contain any additive noise.
Integrated Gradients is designed for the \textit{Completeness}. Thus, the prediction $f(\mathbf{z}_0)$ is equivalent to the integration, and this attribution method satisfies the NIB condition.

\section{The MIA score of GBA methods according to iteration}
In this section, we describe the MIA score over the entire iteration ($i=19$).
\begin{table}[ht!]
    \caption{
    The MIA $Q$ comparison of each attribution methods for NIB VGG16 on CIFAR-10. Iteration $i$ denotes the number of applying the attribution mask to the input features.
    }
    \centering
	\resizebox{0.70\textwidth}{!}{
        \begin{tabular}{|c|c|c|c|c|c|c|c|c|}
        \hline
        \multicolumn{3}{|c|}{Dataset}                                     & \multicolumn{6}{c|}{CIFAR-10}                                                                     \\ \hline
        \multicolumn{3}{|c|}{Method}                                      & PGxI  & GBP   & GxI   & DL    & IG    & \begin{tabular}[c]{@{}c@{}}GxSI\\ (Proposed)\end{tabular} \\ \hline
        \multirow{20}{*}{Iteration} & \multicolumn{2}{c|}{$i=0$}          & \multicolumn{6}{c|}{91.50}                                                                        \\ \cline{2-9} 
                                    & $i=1$  & \multirow{19}{*}{$Q$ (\%)} & 80.60 & 85.80 & 90.20 & 90.20 & 90.20 & 94.10                                                     \\ \cline{2-2} \cline{4-9} 
                                    & $i=2$  &                            & 64.50 & 69.80 & 96.30 & 96.30 & 96.30 & 98.20                                                     \\ \cline{2-2} \cline{4-9} 
                                    & $i=3$  &                            & 49.50 & 39.10 & 93.90 & 93.90 & 93.90 & 99.30                                                     \\ \cline{2-2} \cline{4-9} 
                                    & $i=4$  &                            & 36.70 & 19.50 & 91.70 & 91.70 & 91.70 & 99.60                                                     \\ \cline{2-2} \cline{4-9} 
                                    & $i=5$  &                            & 28.70 & 13.70 & 87.00 & 87.00 & 87.10 & 99.70                                                     \\ \cline{2-2} \cline{4-9} 
                                    & $i=6$  &                            & 24.30 & 12.10 & 82.80 & 82.80 & 82.70 & 99.70                                                     \\ \cline{2-2} \cline{4-9} 
                                    & $i=7$  &                            & 21.70 & 11.70 & 76.80 & 76.80 & 76.90 & 99.80                                                     \\ \cline{2-2} \cline{4-9} 
                                    & $i=8$  &                            & 20.30 & 11.20 & 72.00 & 72.00 & 72.20 & 99.80                                                     \\ \cline{2-2} \cline{4-9} 
                                    & $i=9$  &                            & 19.40 & 11.00 & 67.50 & 67.40 & 67.60 & 99.80                                                     \\ \cline{2-2} \cline{4-9} 
                                    & $i=10$ &                            & 19.00 & 10.90 & 63.60 & 63.70 & 63.60 & 99.80                                                     \\ \cline{2-2} \cline{4-9} 
                                    & $i=11$ &                            & 18.80 & 10.90 & 60.60 & 60.60 & 60.40 & 99.90                                                     \\ \cline{2-2} \cline{4-9} 
                                    & $i=12$ &                            & 18.70 & 10.90 & 59.00 & 59.20 & 58.80 & 99.90                                                     \\ \cline{2-2} \cline{4-9} 
                                    & $i=13$ &                            & 18.50 & 11.00 & 57.40 & 57.30 & 57.60 & 99.80                                                     \\ \cline{2-2} \cline{4-9} 
                                    & $i=14$ &                            & 18.50 & 10.90 & 56.20 & 55.90 & 55.90 & 99.90                                                     \\ \cline{2-2} \cline{4-9} 
                                    & $i=15$ &                            & 18.50 & 10.90 & 54.80 & 54.60 & 54.70 & 99.90                                                     \\ \cline{2-2} \cline{4-9} 
                                    & $i=16$ &                            & 18.50 & 10.90 & 53.80 & 53.90 & 53.70 & 99.80                                                     \\ \cline{2-2} \cline{4-9} 
                                    & $i=17$ &                            & 18.60 & 10.90 & 53.00 & 53.00 & 53.20 & 99.80                                                     \\ \cline{2-2} \cline{4-9} 
                                    & $i=18$ &                            & 18.50 & 10.90 & 52.50 & 52.40 & 52.60 & 99.90                                                     \\ \cline{2-2} \cline{4-9} 
                                    & $i=19$ &                            & 18.30 & 10.90 & 51.90 & 52.00 & 51.90 & 99.80                                                     \\ \hline
        \end{tabular}
     }
\end{table}
\begin{table}[ht!]
    \caption{
    The MIA $Q$ comparison of each attribution methods for NIB VGG16 on CIFAR-100. Iteration $i$ dentoes the number of applying the attribution mask to the input features.
    }
    \centering
	\resizebox{0.70\textwidth}{!}{
        \begin{tabular}{|c|c|c|c|c|c|c|c|c|}
        \hline
        \multicolumn{3}{|c|}{Dataset}                                     & \multicolumn{6}{c|}{CIFAR-100}                                                                    \\ \hline
        \multicolumn{3}{|c|}{Method}                                      & PGxI  & GBP   & GxI   & DL    & IG    & \begin{tabular}[c]{@{}c@{}}GxSI\\ (Proposed)\end{tabular} \\ \hline
        \multirow{20}{*}{Iteration} & \multicolumn{2}{c|}{$i=0$}          & \multicolumn{6}{c|}{69.80}                                                                        \\ \cline{2-9} 
                                    & $i=1$  & \multirow{19}{*}{$Q$ (\%)} & 50.50 & 61.40 & 57.50 & 57.50 & 57.50 & 66.20                                                     \\ \cline{2-2} \cline{4-9} 
                                    & $i=2$  &                            & 30.10 & 38.20 & 64.10 & 64.10 & 64.10 & 75.90                                                     \\ \cline{2-2} \cline{4-9} 
                                    & $i=3$  &                            & 18.20 & 13.50 & 55.60 & 55.60 & 55.80 & 81.30                                                     \\ \cline{2-2} \cline{4-9} 
                                    & $i=4$  &                            & 11.60 & 0.04  & 47.30 & 47.30 & 47.40 & 83.10                                                     \\ \cline{2-2} \cline{4-9} 
                                    & $i=5$  &                            & 0.08  & 0.02  & 39.90 & 39.90 & 40.00 & 83.80                                                     \\ \cline{2-2} \cline{4-9} 
                                    & $i=6$  &                            & 0.06  & 0.01  & 33.80 & 33.90 & 34.00 & 85.20                                                     \\ \cline{2-2} \cline{4-9} 
                                    & $i=7$  &                            & 0.05  & 0.01  & 29.00 & 29.00 & 29.00 & 86.60                                                     \\ \cline{2-2} \cline{4-9} 
                                    & $i=8$  &                            & 0.04  & 0.01  & 24.90 & 25.00 & 25.10 & 87.60                                                     \\ \cline{2-2} \cline{4-9} 
                                    & $i=9$  &                            & 0.04  & 0.01  & 21.30 & 21.30 & 21.00 & 88.10                                                     \\ \cline{2-2} \cline{4-9} 
                                    & $i=10$ &                            & 0.04  & 0.01  & 18.90 & 18.90 & 18.50 & 89.00                                                     \\ \cline{2-2} \cline{4-9} 
                                    & $i=11$ &                            & 0.04  & 0.01  & 16.80 & 16.70 & 16.60 & 90.00                                                     \\ \cline{2-2} \cline{4-9} 
                                    & $i=12$ &                            & 0.04  & 0.01  & 14.60 & 14.50 & 14.70 & 90.20                                                     \\ \cline{2-2} \cline{4-9} 
                                    & $i=13$ &                            & 0.04  & 0.01  & 13.50 & 13.30 & 13.20 & 89.80                                                     \\ \cline{2-2} \cline{4-9} 
                                    & $i=14$ &                            & 0.04  & 0.01  & 12.50 & 12.30 & 12.00 & 89.80                                                     \\ \cline{2-2} \cline{4-9} 
                                    & $i=15$ &                            & 0.04  & 0.01  & 11.50 & 11.40 & 11.40 & 90.00                                                     \\ \cline{2-2} \cline{4-9} 
                                    & $i=16$ &                            & 0.03  & 0.01  & 10.90 & 11.00 & 10.50 & 89.70                                                     \\ \cline{2-2} \cline{4-9} 
                                    & $i=17$ &                            & 0.04  & 0.01  & 10.10 & 10.20 & 10.40 & 89.70                                                     \\ \cline{2-2} \cline{4-9} 
                                    & $i=18$ &                            & 0.04  & 0.01  & 0.10  & 0.10  & 0.10  & 89.60                                                     \\ \cline{2-2} \cline{4-9} 
                                    & $i=19$ &                            & 0.04  & 0.01  & 0.10  & 0.10  & 0.10  & 89.70                                                     \\ \hline
        \end{tabular}
     }
\end{table}

\section{The MIA score of GxSI methods according to model configurations.}

In this section, we describe the MIA score over the entire iteration ($i=9$) according to model configurations.

\begin{table}[ht!]
    \caption{
    The MIA of GxSI comparison of each model architecture on CIFAR-10. ``\# of layer" denotes the number of convolutional layers and linear layers. Iteration $i$ denotes the number of applying the attribution mask to the input features.
    }
    \centering
	\resizebox{1.0\textwidth}{!}{
        \begin{tabular}{|c|c|c|c|c|c|c|c|}
        \hline
        \multicolumn{3}{|c|}{Name}                                      & NIB VGG13-6 & NIB VGG13-4 & NIB VGG13-2 & NIB VGG13 & NIB VGG16 \\ \hline
        \multicolumn{3}{|c|}{\# of layer}                               & 7           & 9           & 11          & 13        & 16        \\ \hline
        \multirow{10}{*}{Iteration} & \multicolumn{2}{c|}{$i=0$}        & 83.3        & 89.0        & 90.9        & 91.4      & 91.5      \\ \cline{2-8} 
                                    & $i=1$ & \multirow{9}{*}{$Q (\%)$} & 98.4        & 97.5        & 97.2        & 96.5      & 94.1      \\ \cline{2-2} \cline{4-8} 
                                    & $i=2$ &                           & 99.9        & 99.7        & 99.6        & 99.1      & 98.2      \\ \cline{2-2} \cline{4-8} 
                                    & $i=3$ &                           & 100.0       & 100.0       & 99.9        & 99.8      & 99.3      \\ \cline{2-2} \cline{4-8} 
                                    & $i=4$ &                           & 100.0       & 100.0       & 100.0       & 99.9      & 99.6      \\ \cline{2-2} \cline{4-8} 
                                    & $i=5$ &                           & 99.9        & 100.0       & 100.0       & 99.9      & 99.7      \\ \cline{2-2} \cline{4-8} 
                                    & $i=6$ &                           & 99.7        & 100.0       & 100.0       & 100.0     & 99.7      \\ \cline{2-2} \cline{4-8} 
                                    & $i=7$ &                           & 99.4        & 100.0       & 100.0       & 99.9      & 99.8      \\ \cline{2-2} \cline{4-8} 
                                    & $i=8$ &                           & 98.9        & 100.0       & 100.0       & 99.9      & 99.8      \\ \cline{2-2} \cline{4-8} 
                                    & $i=9$ &                           & 98.3        & 100.0       & 100.0       & 99.9      & 99.8      \\ \hline
        \end{tabular}
     }
\end{table}

\section{Additional visualization of masked input features}
We provide representations of the masked input features with each attribution method for NIB VGG16. The features are sampled on test set of CIFAR-10 .  

\begin{figure}[ht!]
	\centering
	\includegraphics[scale=0.3]{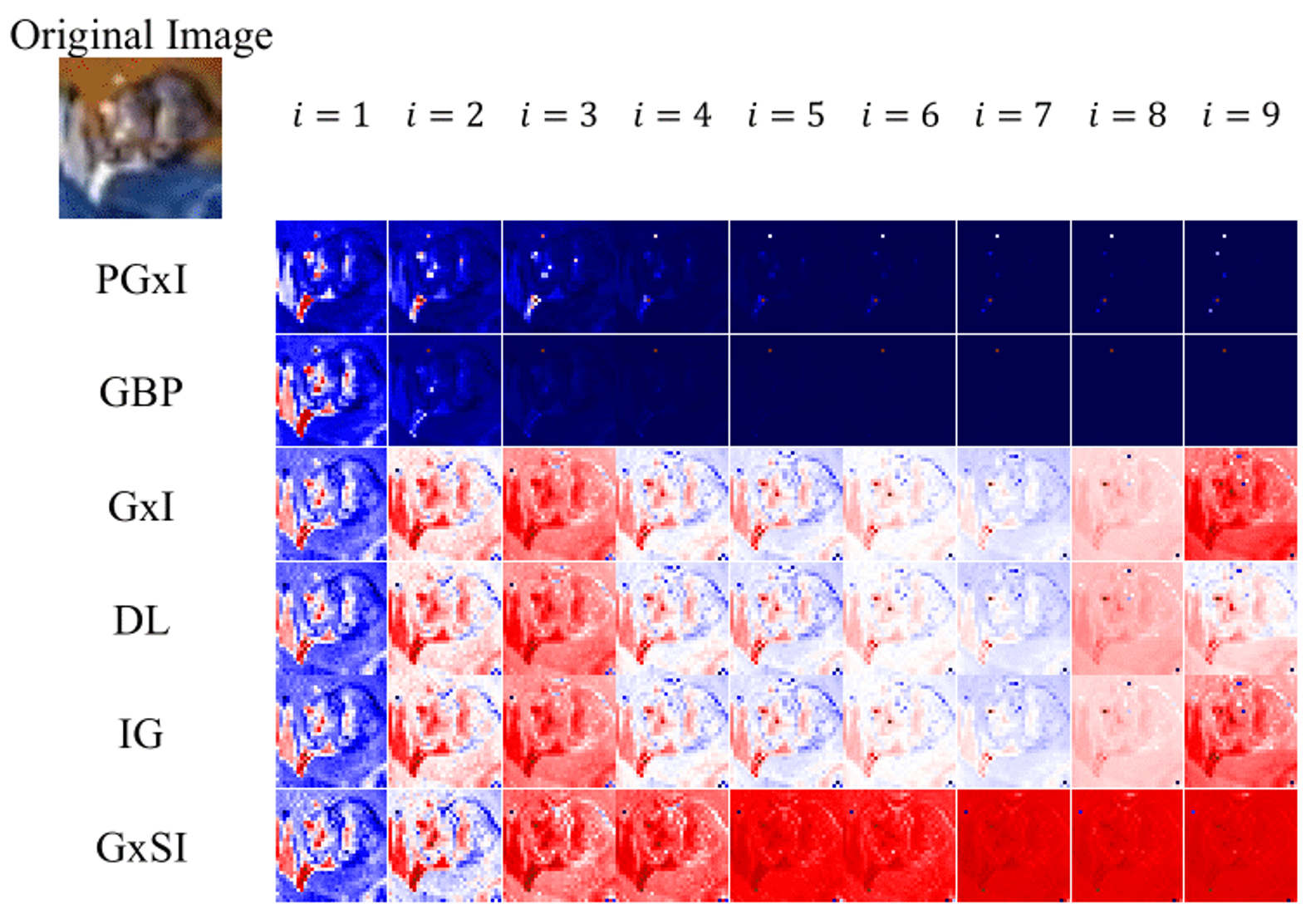}
	\caption{
	The features are masked for cat example and are visualized by summing the channel. Red color denotes high value and blue color is low value.
	}
\end{figure}
\begin{figure}[ht!]
	\centering
	\includegraphics[scale=0.3]{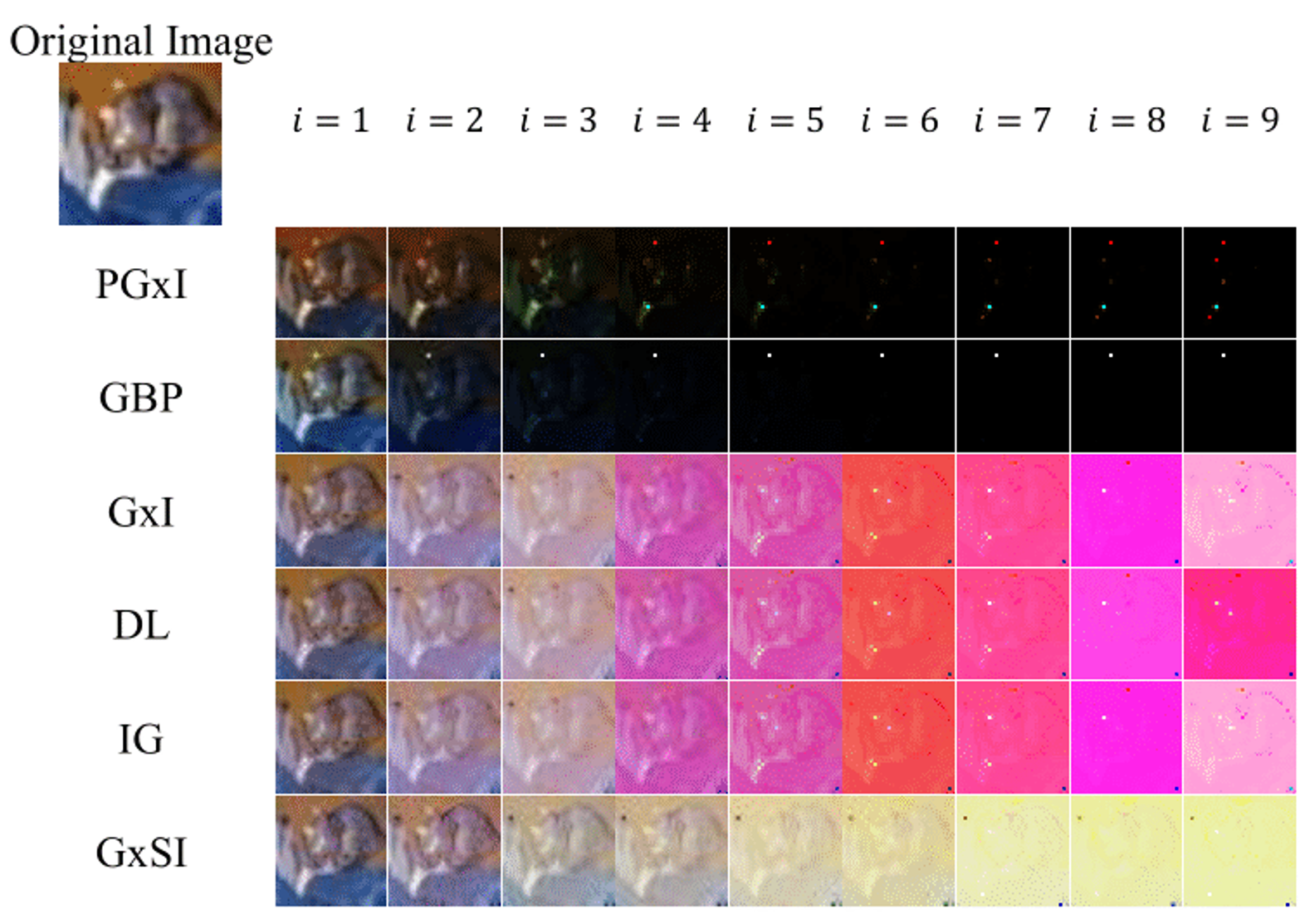}
	\caption{
	The features are masked for cat example and are visualized with each RGB channel. 
	}
\end{figure}
\begin{figure}
	\centering
	\includegraphics[scale=0.3]{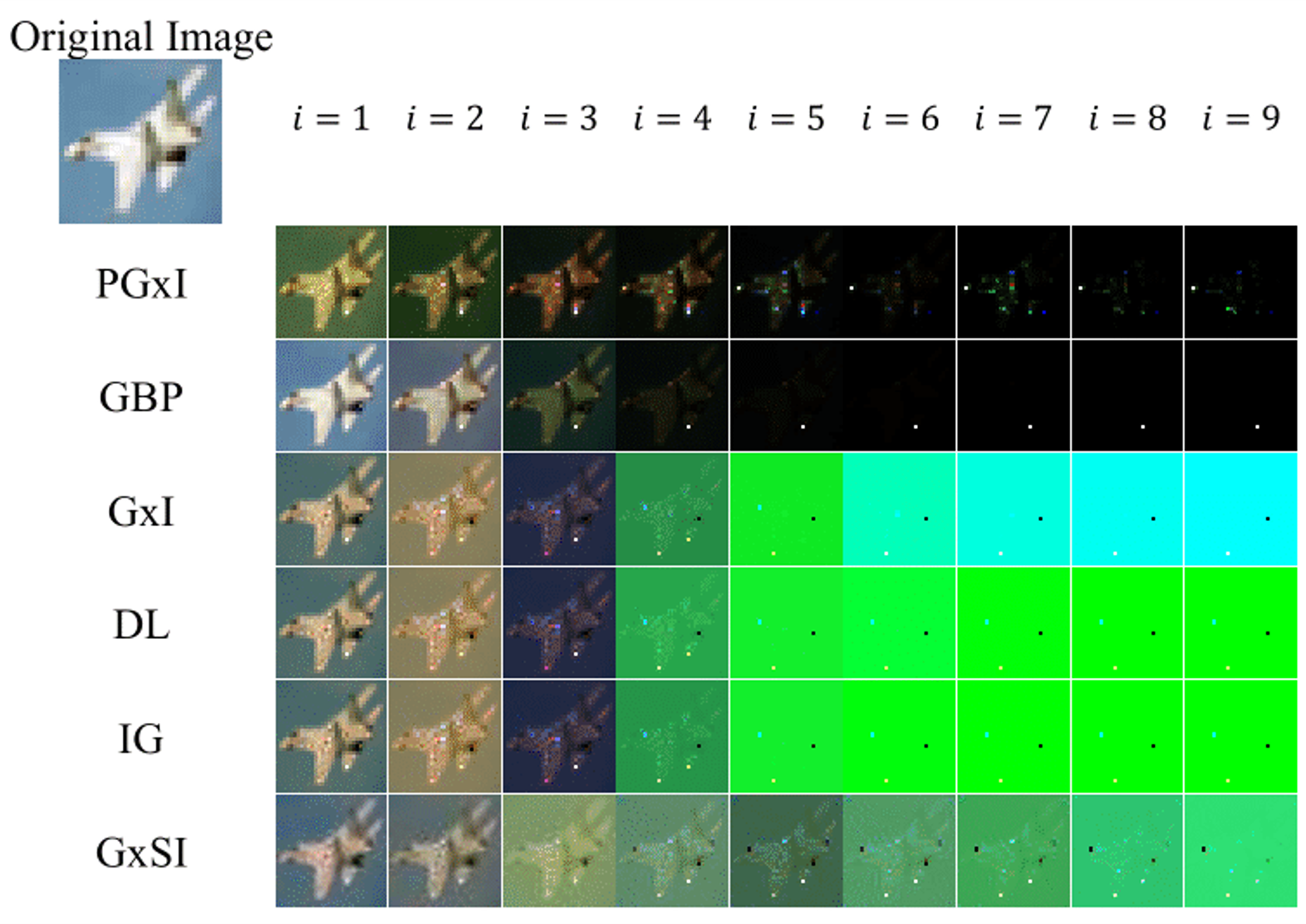}
	\caption{The features are masked for airplane example and are visualized with each RGB channel.}
\end{figure}
\begin{figure}
	\centering
	\includegraphics[scale=0.3]{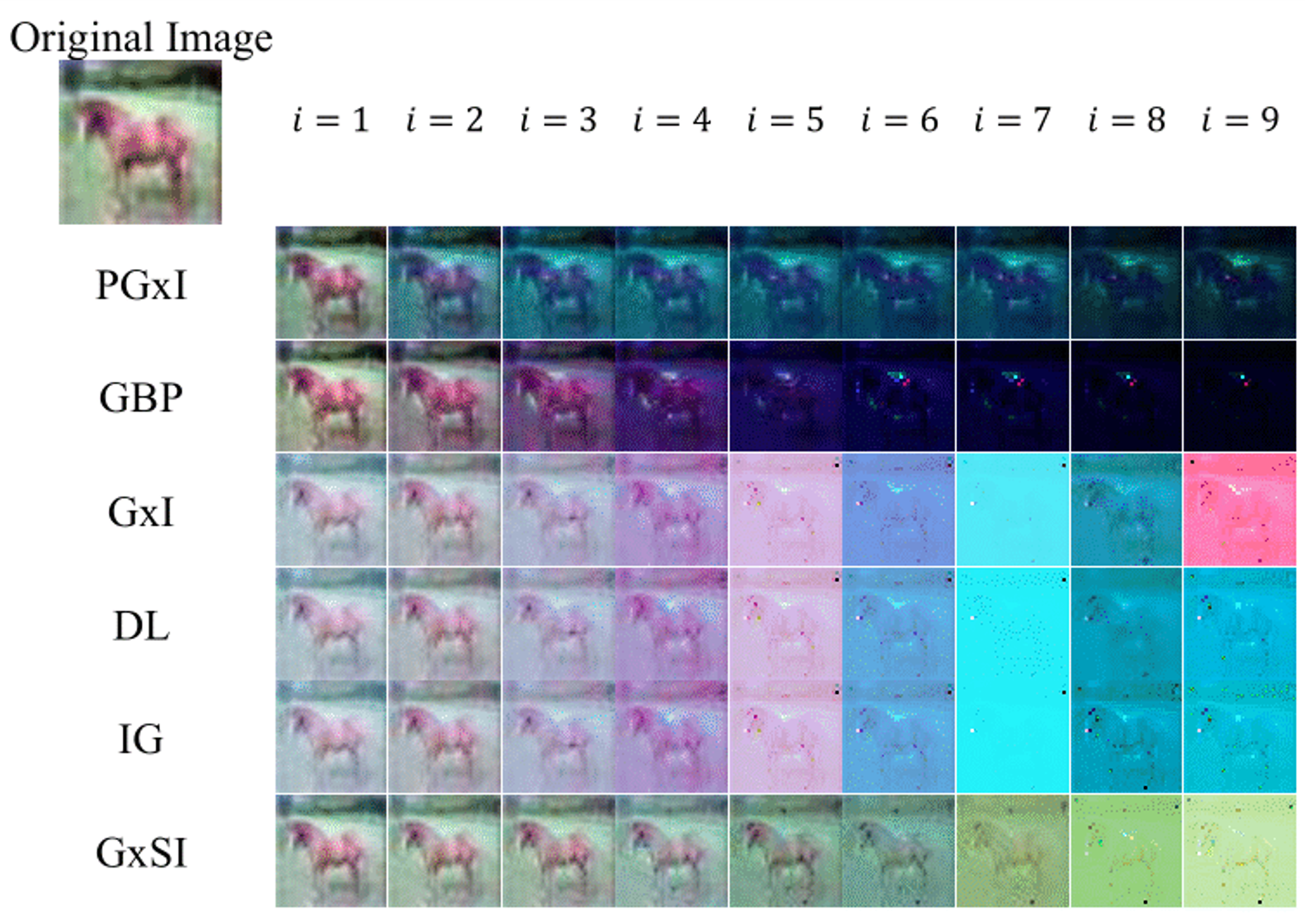}
	\caption{The features are masked for horse example and are visualized with each RGB channel.}
\end{figure}
\begin{figure}
	\centering
	\includegraphics[scale=0.3]{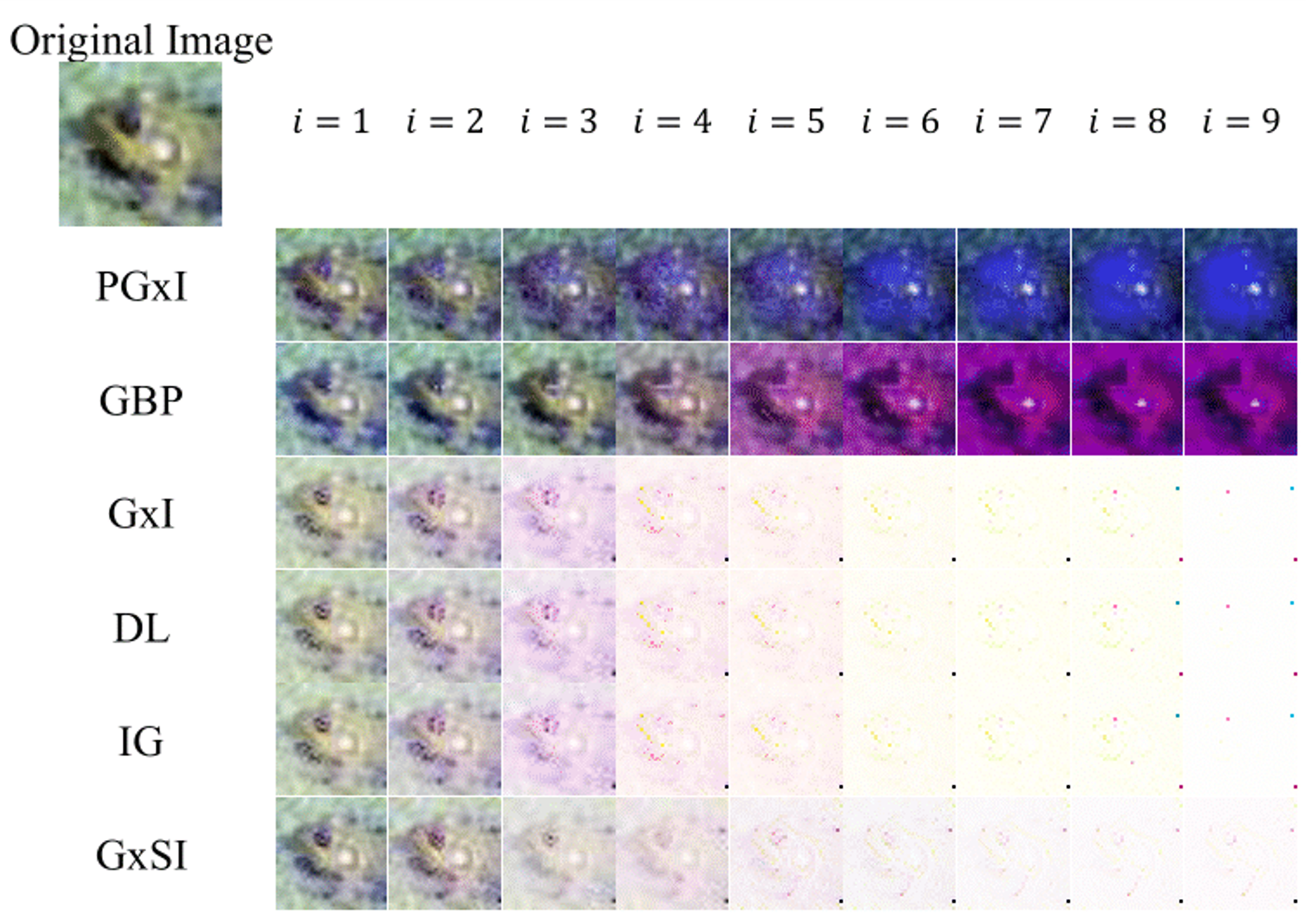}
	\caption{The features are masked for flog example and are visualized with each RGB channel.}
\end{figure}
\begin{figure}
	\centering
	\includegraphics[scale=0.3]{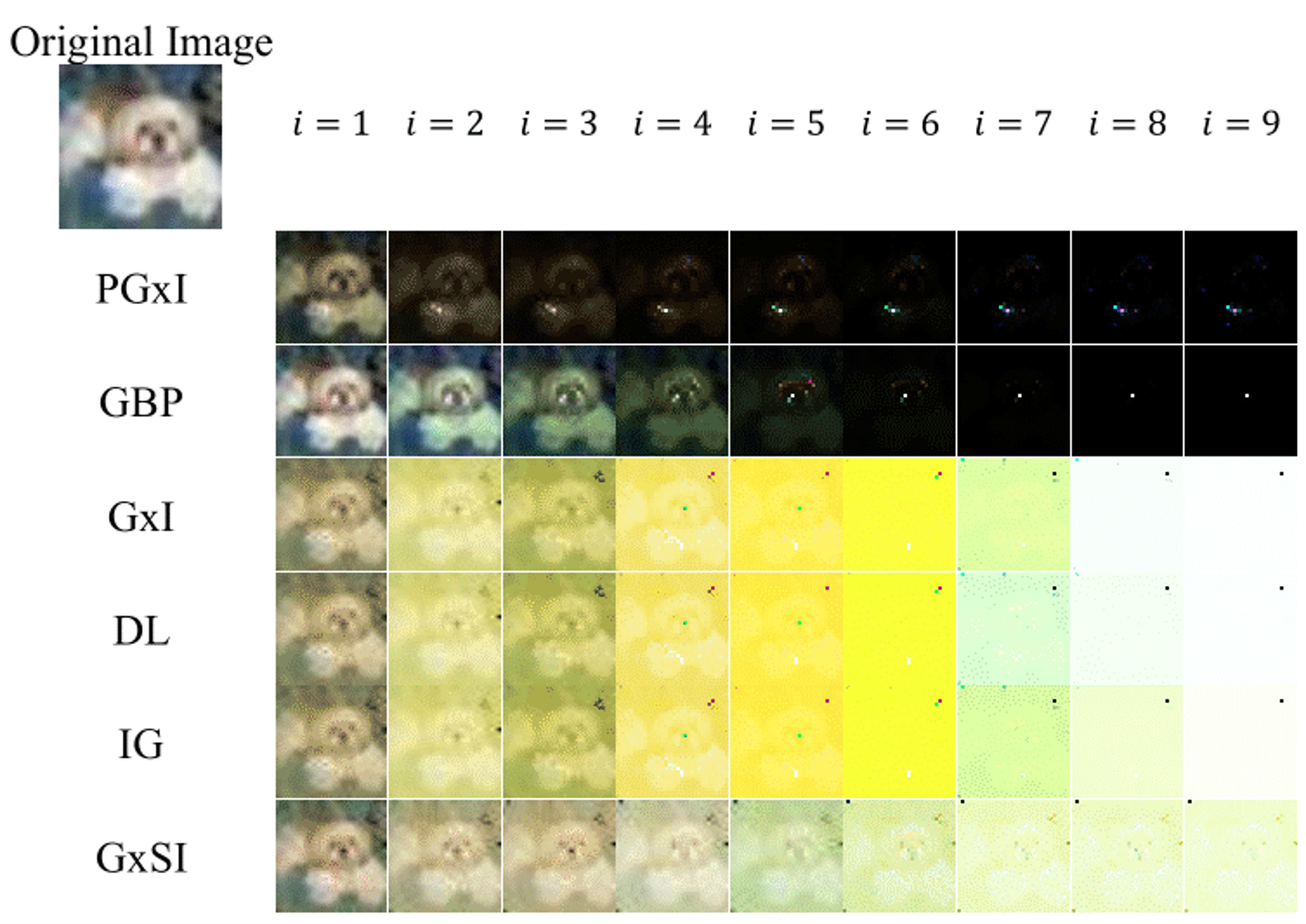}
	\caption{The features are masked for dog example and are visualized with each RGB channel.}
\end{figure}

\end{document}